%% file: main.tex
\documentclass[twoside,11pt]{article}

\usepackage{jmlr2e}
\usepackage{todonotes}

\usepackage{listings}
\usepackage{amssymb}
\usepackage{framed} 
\usepackage{amsmath}
\usepackage{amssymb}
\usepackage{relsize}
\usepackage{tikz}
\usepackage{xr}
\usepackage{times}
\usepackage{graphicx} % more modern
\usepackage{subfigure} 
\usepackage{wrapfig}
\usepackage{algorithm}
\usepackage{algorithmic}
\usepackage{tikz-cd}

\usetikzlibrary{backgrounds}
\usetikzlibrary{calc}
\usepackage{soul}
\usepackage{stackrel}

\usepackage{relsize}

\tikzset{fontscale/.style = {font=\relsize{#1}}
    }

\input{sections/macros}

\jmlrheading{1}{2019}{}{9/20}{[Under Review]}{wguss19b}{William Guss and Ruslan Salakhutdinov}

\ShortHeadings{Universal Approximation}{Guss and Salakhutdinov}
\firstpageno{1}

\begin{document}

\title{On Universal Approximation by Neural Networks with Uniform Guarantees on Approximation of Infinite Dimensional Maps}

\author{\name William H.\ Guss \email wguss@cs.cmu.edu \\
       \addr Machine Learning Department\\
       Carnegie Mellon University\\
       Pittsburgh, PA 15206, USA
       \AND
       \name Ruslan Salakhutdinov \email rsalakhu@cs.cmu.edu \\
       \addr Machine Learning Department\\
      Carnegie Mellon University\\
       Pittsburgh, PA 15206, USA}

% \editor{Neil Lawrence and Mark Reid}

\maketitle

\begin{abstract}%   <- trailing '%' for backward compatibility of .sty file

The study of universal approximation of arbitrary functions $f: \X\to \Y$ by neural networks has a rich and thorough history dating back to \cite{kolmogorov1957representation}. 
In the case of learning finite dimensional maps, many authors have shown various forms of the universality of both fixed depth and fixed width neural networks. 
However, in many cases, these classical results fail to extend to the recent use of approximations of neural networks with infinitely many units for functional data analysis, dynamical systems identification, and other applications where either $\X$ or $\Y$ become infinite dimensional. Two questions naturally arise: which infinite dimensional analogues of neural networks are sufficient to approximate any map $f: \X \to \Y$, and when do the finite approximations to these analogues used in practice approximate $f$ uniformly over its infinite dimensional domain $\X$?

In this paper, we answer the open question of universal approximation of nonlinear operators when $\X$ and $\Y$ are both infinite dimensional. We show that for a large class of different infinite analogues of neural networks, any continuous map can be approximated arbitrarily closely with some mild topological conditions on $\X$. Additionally, we provide the first lower-bound on the minimal number of input and output units required by a finite approximation to an infinite neural network to guarantee that it can uniformly approximate any nonlinear operator using samples from its inputs and outputs.

\end{abstract}


\input{sections/introduction}

\input{sections/problem_setup}

\input{sections/related_works}

\input{sections/proofs}

\section{Conclusion}

% Theorem \ref{thm:sample_bounds}
% Corollary \ref{col:lip}

% Summary
In this work we answer two open questions of universal approximation for neural networks in the affirmative when their input and output domain become infinite dimensional: in particular our main results assert that in the setting of nonlinear operator approximation and nonlinear basis map approximation, \emph{several classes of  two layer neural networks are universal}. 

To show these results we developed a category theoretic proof technique centered around objects called sample factorizations. For familiar spaces such as $C(K)$ these objects are pairs of sampling and interpolation processes. We show that sample factorizations have a functorial property that lets us reduce the problem of approximating an infinite dimensional map to a finite dimensional one with nice commutative properties. By proving that the infinite dimensional analogues of neural network layers are universal with respect to these sample factorizations, we are able to leverage the classical universal approximation results for neural networks to show universality in the infinite dimensional case. 

As an additional upshot of this technique we give provide the first upper bound on the minimum number of input and output units required to guarantee that a finite neural networks is capable of universally approximating a nonlinear operator uniformly over the function space on which it operates. Such a guarantee is not possible using classical universal approximation results alone. Specifically, we show that this minimal architecture depends on covering number of its domain with balls whose radius is the ratio between the regularity of the desired operator and the regularity of the functions on which it operates.

Our results suggest directions for future work. First, we show that infinite dimensional neural networks composed of operator layers are universal approximators of nonlinear operators, and in addition we provide sufficient conditions for some finite dimensional neural network to approximate an infinite dimensional one. A natural next question is: given some regularity conditions on a nonlinear operator, are the weight functions of the operator layer based neural networks that approximate it smooth or regular? An answer in the affirmative yields a new method for parameterizing  finite neural networks by learning smooth estimators of these weight functions with upper bounds on how many samples are needed to achieve universal approximation. Second, the uniform universality relationship between finite dimensional and infinite dimensional neural networks opens the possibility of approaching the problem of non convex optimization of neural networks using techniques from the calculus of variations on their infinite dimensional analogues to aid in convergence results. Finally, there is a question of when our upper bound on the minimum number of input and output units required for uniformity can be strengthened. In particular, we leave the characterization of domains and function spaces for which the bound becomes sub-exponential to future work.

\vskip 0.2in

\bibliography{main}

\input{sections/appendix}

\end{document}

%% file: sections/macros.tex
\def\dim{\operatorname{dim}\,}

\newcommand{\norm}[1]{ \left\|  #1 \right\|}
\newcommand{\absn}[1]{ \left|  #1 \right|}

\def\t2{\tfrac12}

\def\cl{\text{cl}}

\def\scriptf{{\mathcal F}}

\def\scripth{{\mathcal H}}
\def\scriptd{{\mathcal D}}

\def\F{{\mathcal F}}
\def\X{{\mathcal X}}
\def\Y{{\mathcal Y}}
\def\G{{\mathcal G}}
\def\C{{\mathcal C}}

\def\H{{\mathcal H}}
\def\L{{\mathcal L}}
\def\cl{\operatorname{cl}\,}
\def\diam{\operatorname{diam}\,}
\def\id{\operatorname{id}}
\def\R{{\mathbb R}}
\def\Deltab{{\bar{\Delta}}}

\def\Af{{Q}}

\def\To{{T^\mathfrak{o}}}
\def\Tf{ {T^\mathfrak{f}}}
\def\Tn{{T^\mathfrak{n}}}
\def\Tb{{T^\mathfrak{b}}}

\newcommand{\blue}[1]{{\color{blue} #1}}
\newcommand{\red}[1]{{\color{red} #1}}
\newcommand{\orange}[1]{{\emph{\color{blue} #1}}}

\newenvironment{proofsubsection}[1]{\subsection{Proof of #1}}{\hfill$\blacksquare$}
\newenvironment{proofsubsubsection}[1]{\noindent \newline\textbf{Proof of #1}}{\hfill$\blacksquare$}

%% file: sections/introduction.tex
%!TEX root = ../main.tex
\section{Introduction}

Consider the problem of universal approximation of maps between topological vector spaces using neural networks. In particular, let $\F  = \{f : \X \to \Y\}$ be some family of morphisms, and let $\G_L = \{ T_{L} \circ g \circ T_{L-1} \circ \cdots g \circ  T_1: \X \to \Y\}\subset \F$ be a family of neural networks defined by repeated composition of a series of affine maps $(T_{\ell})_{\ell=1}^L$ and a point-wise activation function $g: \R \to \R$. \emph{For a given topology $\mathcal{T}$  on $\F$ and $E\subset \X$, when is it true that $\cl(\G_L|_{E}) = \F|_E$, where $\cl(\cdot)$ is the closure in $\mathcal{T}$?}

An affirmative answer to this question when $\X$ and $\Y$ are finite dimensional is essential to the use of neural networks in many standard applications, for example in decision theory, where the use of neural networks as a \emph{practical} parametric model hinges on their ability to approximate any measurable $f: \R^n \to \R$. 
    Fortunately, it has been shown that both arbitrary width neural networks of \blue{\emph{fixed depth}} and arbitrary depth neural networks of \red{\emph{fixed width}} are universal approximators.
      In the case of \emph{\blue{fixed depth}}, $L=2$ suffices:  The initial work of \cite{univapprox}, \cite{hornik1991approximation}, and \cite{funahashi1989approximate} showed that surprisingly, $\cl(\G_2|_K) = C(K, \Y)$ in the uniform topology where $T_\ell$ unrestricted, $g$ is "sigmoidal" and $K \subset \X$ is any compact set.
        % Further work by [HSW, SW, HSW (1990)] extend these results to the case where $g$ is non-sigmoidal and more generally where  $\scriptf = C^p(K,\Y)$ is the set of $p$-differentiable functions with the uniform topology over all derivatives. % Related works!?
      Likewise in the case of \emph{\red{fixed width}}: \cite{lu2017expressive} showed that when the number of hidden units is restricted as $\dim T_\ell \leq \dim \X + 4$ for all $\ell$, $\cl(\bigcup_{\ell=1}^\infty \G_\ell) = L^1(\X, \Y)$ in the usual topology.

  However, when either $\X$ or $\Y$ are taken to be infinite dimensional substantially less is known.  
    The question of universal approximation of mappings between such spaces is of particular interest for the use of neural networks in settings whereby the data or labels are functional in nature 
      such as  
        dynamical systems \cite{chen1995universal}, % Do we need this?> 
        inverse mapping problems \cite{adler2017solving}, 
        and functional data analysis (\cite{ramsay2004functional}, \cite{besse2000autoregressive}). 
    In such cases, one considers whether or not some $\G_L$ is expressive enough to learn nonlinear functionals, operators, or basis maps where $(T_\ell)_{\ell=1}^L$ are infinite dimensional analogues of the finite affine maps considered above. 
    To that end, \cite{stinchcombe1999neural} first showed that for the case of \orange{nonlinear functionals} if $\X = C(K)$ and $\Y =\R^d$ then $\G_2$ is universal; that is, $\cl(\G_2|_E) = C(C(K), \R^d)|_E$ in the uniform topology when $E \subset \X,K\subset \R^d$ compact, $T_1: f \mapsto \int_{K} f(u) w(u)\; d\mu(u) + b$ and $T_2$ finite dimensional. In the setting of \orange{nonlinear operators}, where $\X$ and $\Y$ become infinite dimensional,  the picture is less complete. \cite{chen1995universal} studied this problem in the context of dynamical systems and show an interesting theorem that the function class $\scripth$ of pointwise products of nonlinearities composed with affine maps is universal 
       that is, $\cl(\scripth|_E) = C(C(K), C(K'))|_E$ for a compact $E$.
    However \emph{it is still not known }whether or $\G_L$ is dense in this space. Lastly, the case of \orange{nonlinear basis maps}, where $\X$ is finite and $\Y$ is infinite dimensional, is particularly important to the use of neural processes \cite{garnelo2018neural} and other generative models of functions whereby a finite dimensional latent variable is used to generate functions. Unfortunately, when $\X = \R^d$ and $\Y = C(K')$ the question of whether or not $\cl(\G_L|_E) = C(\R^d, C(K'))$ remains open.

  \paragraph{Our Contributions.} In this paper, we present several new results which answer the open questions of universal approximation for nonlinear operators and nonlinear basis maps in the affirmative. In particular we show that one only needs that $\G_L$ consist of two layer infinite dimensional neural networks, where summation becomes integration, to show that $\cl(\G_2|_E) = C(C(K), C(K'))|_E$. Further we develop techniques that improve this result to an even milder set of architectures: specifically, when $\G_L$ is restricted to two layer neural networks with infinitely many input/output units but only \emph{finitely many hidden units}, we again have $\cl(\G_2|_E) = C(C(K), C(K'))|_E$. In our last universality result, we show that when $\G_L$ is restricted to two layer neural networks with finitely many input and hidden units but infinitely many output units can approximate any parameterization of a (stochastic) process; that is, $\cl(\G_2|_E) = C(\R^d, C(K'))$. Finally, as a direct result of the underlying proof techniques we develop, we provide upper bounds on the number of input/output units needed to uniformly approximate a nonlinear operator using a standard, finite fully connected neural network. To the best of our knowledge, this work is the first to establish such bounds.

%% file: sections/problem_setup.tex
\section{Preliminaries}

We begin by formally defining the various families of neural networks and topological spaces in which we wish to answer the question of universal approximation. In particular, when $\X$ or $\Y$ are infinite dimensional, there are many possible types of affine maps one can use to construct the layers of a neural network $G \in \G_L$. The proof techniques of this paper will allow us to show universality in three natural settings  resulting from combinations of the layer types given in Definition \ref{def:layer_types} below.

Denote the set of $L_p$-integrable functions with respect to the Lebesgue measure $\mu$ from a space $K \subset \R^d$ to $\R$ as $L_p(K)$, and let $C(X,Y)$ be the set of continuous functions between $X$ and $Y$. We further adopt the notation that $C(X) = C(X,\R)$. Finally let $\|\cdot\|_\X$ denote the norm associated to $\X$ which induces its topology. Unless otherwise stated we endow $C(\cdot, \cdot)$ with the uniform topology induced by the supremum norm $\|\cdot\|_\infty$ and $L_p(\cdot, \cdot)$ with the usual topology induced by its norm  $\|\cdot\|_{L_p}$. The layers considered are defined as follows.
\begin{definition}[Infinite Dimensional Layers]\label{def:layer_types}
  Let $H, H'$ be topological vector spaces, let $K \subset \R^d$,\\$ K' \subset \R^{d'}$, and let $T: H \to H'$ be some affine map
  \begin{enumerate}
    \item If $H = L_1(K), H' = L_1(K')$, then $T$ denoted  $\To$ is said to be an \textbf{operator layer} if there is some (weak$^*$) continuous family of measures $W_t \ll \mu$ over $t \in K'$ with Radon-Nikodym derivative  $w(u,t)$   and a function $b \in L_1(K)$ such that 
    \begin{equation}
      \To: \xi \mapsto \left(t \mapsto \int_{K} \xi(u) w(u,t)\;d\mu + b(t)\right).
    \end{equation}
    as presented in \cite{rossi2002theoretical} in less generality.
    \item If $H = L_1(K), H' = \R^{d'}$, then $T$ denoted  $\Tf$ is said to be an \textbf{functional layer} if there is some measure $W \ll \mu$ with Radon-Nikodym derivative $w(u)$ and vector $b \in \R^d$ such that $\Tf: \xi \mapsto \left\langle \xi, w\right\rangle_{L_1(K)} +b$ as first introduced by \cite{stinchcombe1999neural}.
    \item If $H =\R^d, H' = L_1(K')$, then $T$ denoted  $\Tb$ is said to be an \textbf{basis layer} if there is some function in $L_1(K', \R^d)$ and $b \in L_1(K')$ such that $\Tb: y \mapsto \left( t \mapsto \left\langle y, w(t)\right\rangle_{\R^d} +b(t) \right)$.
    \item When either $H$ or $H'$ are finite dimensional, we yield the standard fully-connected layer, denoted $\Tn$
  \end{enumerate}
\end{definition}

The layer types of Definition \ref{def:layer_types}(1-3) are very natural candidates for building universal approximators. 
  For example, operator layers, as first presented in \cite{rossi2002theoretical}, arise when considering the limit of a neural network as its number of hidden and input units approaches infinity and some regularity conditions are imposed on its weights.
  Likewise the functional and basis layers of \cite{rossi2002theoretical} and \cite{le2007continuous} are a result of a similar limiting process. One might hope that so long as the map $F \in \scriptf$ desired to be estimated is in some vague sense the limit of a finite dimensional process, the respective $G \in \G_L$ should maintain universality. As we will show, the conditions on $\X$, $\Y$, and $\F$ under which this intuition results in an affirmitive answer are actually quite mild; as in the original universal approximation results of \cite{univapprox} for finite neural networks, \emph{continuity of $F$ and compactness of its domain is all you need.}

\vspace{-0.3cm}

\section{Main Results}

We now provide several main results of universal approximation using the proof techniques developed in Section \ref{sec:proofs}. 
  Let $K, K' \subset \R^{d}, \R^{d'}$ be compact domains
  , and let $E, E' \subset C(K), C(K')$ be compact families of functions over $K$ and $K'$ respectively.  
  Further let $g: \R \to \R$ be any continuous, non-polynomial activation function. We turn our attention to the open question of universality when $\X$ and $\Y$ are both infinite dimensional.

  \begin{theorem} \label{thm:operator} Let $F: C(K) \to C(K')$ be continuous. For every $\epsilon > 0$ and any $d'' >0$, there exists a compact $K'' \subset \R^{d''}$, two continuous families of Lebesgue absolutely continuous measures $(W^1_v \ll \mu)_{v\in K}, (W^2_y \ll \mu)_{y\in K'}$, and functions $b \in L_1(K), b' \in L_1(K'')$ such that
    \begin{equation}
       \sup_{f \in E, y\in K'} \norm{\int_{K''} g \left(\int_K f(u)\;dW^1_v(u) + b(v) \right)\;dW^2_y(v) +b'(y) - F[f](y)} < \epsilon.
    \end{equation}
    Hence $\cl\left(\G_2|_E = \{\To \circ g \circ \To: C(K) \to C(K')\}\right) = C(C(K), C(K'))|_E$.
  \end{theorem}
  In other words, there are weight functions $w^1(u,v)$ and $w^2(u,v)$ (and biases) which are the limit of the weights of finite neural networks (as in \cite{rossi2002theoretical}) such that two layer neural networks composed of the corresponding operator layers can approximate any continuous, nonlinear operator $F$. 

  It turns out that one can approximate $K$ using an even more restricted class $\G_2$: up to some arbitrary error, a neural network which extracts only a finite dimensional set of latent features from $f \in C(K)$ has enough power to approximate $K$ uniformly.

  \begin{theorem}\label{thm:fb} If $F: C(K) \to C(K')$ is continuous, then for every $\epsilon > 0$ there exists an $N  > 0$ and two finite collections of functions $(w^1_i \in L_1(K))_{i=1}^N$ and $(w^2_i \in C(K'))_{i=1}^N$ and biases $b \in \R^N, b' \in C(K')$ such that
  \begin{equation}
    \sup_{f \in E} \norm{\sum_{i=1}^N w^2_i g\left(\int_K f(u) w^1_i(u)\;d\mu(u) + b_i  \right) + b' - F[f]}_\infty < \epsilon.
  \end{equation}
  Hence $\cl\left(\G_2|_E = \{\Tb \circ g \circ \Tf: C(K) \to C(K')\}\right) = C(C(K), C(K'))|_E$.
  \end{theorem}

  Next, consider the setting where $\X$ is finite dimensional and $\Y$ is not. The following result answers universality of $\G_2$ in the affirmative.
  \begin{theorem}\label{thm:basis} If $F: \R^d \to C(K')$ is continuous, then for every $\epsilon > 0$ and for all $E \subset \R^d$ compact, there exists an $N  > 0$, a matrix $W\in \R^{d\times N}$, a collection of functions $(w^2_i \in C(K'))_{i=1}^N$ and biases $b \in \R^N, b' \in C(K')$ such that
    \begin{equation}
      \sup_{x \in E} \norm{\sum_{i=1}^N w^2_i g\left(W_i^Tx\right) + b' - F[x]}_\infty < \epsilon.
    \end{equation}
    Hence $\cl\left(\G_2|_E = \{\Tn \circ g \circ \Tb: \R^d \to C(K')\}|_E\right) = C(\R^d, C(K'))|_E$.
    \end{theorem}
  Essentially Theorem \ref{thm:basis} states that so long as the number of outputs of a finite neural network approach infinity one can represent any continuous map of $\R^d$ into $C(K)$ uniformly up to some error threshold. In the case of neural processes this guarentees that there exist finite neural networks which can represent any compact distribution over continuous functions uniformly.

\subsection{Uniform Approximation of Operators for Finite Dimensional Neural Networks}
As we will establish in Section \ref{sec:proofs}, the underlying mechanism that asserts Theorems \ref{thm:operator}, \ref{thm:fb}, and \ref{thm:basis} also provides a method for upper bounding the minimum number of  input and output units of a finite neural network that are required to approximate an infinite dimensional nonlinear operator uniformly over its domain and codomain function space.
  In particular, suppose that one attempts to use a finite neural network $\G_L \ni G: \R^{M} \to \R^{M} $ to learn a nonlinear operator $D: \X \to \Y$ (e.g. a dynamical system, transformation of random processes, etc.) by using a fixed set of $M$ sample points of functions $\phi \in \X$ and $D\phi \in \Y$.
       When $G$ is trained over many $\phi$ using data of the form $(\phi(x_i), D[\phi](y_i))_{i=1}^{M}$ resulting from some fixed set of sample points $S_X = (x_i)_{i=1}^M$ and $S_Y = (y_i)_{i=1}^M$, one wonders if using an interpolation of the $M$ outputs of $G(\Delta[\phi])$ to reconstruct $D\phi$ is uniformly accurate over all $\phi \in X$. Further, how many input/output units are required to guarantee uniform accuracy of the reconstruction? The following theorem answers these questions in terms of the covering numbers of $\X$, $\Y$, and the regularity of $D$.

Let $\mathsf{Interp}_{S_Y}: \R^{M} \to \Y$ be any affine map which constructs an interpolation of its input at the points in $S_Y$. Further for any continuous map $f$ defined on a compact set, let $\omega_F(\delta)$ denote its modulus of continuity and $L_F$ denote its Lipschitz constant. Finally let $\C(K, \gamma)$ denote any minimal covering of $K$ with balls of radius $\gamma$.
\begin{theorem}\label{thm:sample_bounds}
  Let $D: C(K) \to C(K)$ be continuous. Then for any $\epsilon >0$ and compact $E\subset C(K)$, then there exist an $N > 0$ matrices $W^1 \in \R^{M\times N}, W^2 \in \R^{N\times M}$ and biases $b \in \R^N, b' \in \R^M$ such that 
  \begin{equation}\label{eq:uniform_sampless}
    \sup_{\phi \in E, y\in K'} \norm{\mathsf{Interp}_{S_Y}\left[W^2g\left(\sum_{i=1}^M W_{ij} \phi(x_i) + b\right) + b'\right](y) - D[\phi](y)} < \epsilon
  \end{equation}
  when the samples $(x_i)$ and $(y_i)$ form covers of $K$: 
  \begin{equation}\label{eq:domain_covering}
    \begin{aligned}
      S_Y = S_X = \C\left(K,  \min\left\{\frac{\psi(\epsilon)}{\ell(\epsilon)}, \frac{\epsilon}{2}\right\}  \right),
    \end{aligned}
  \end{equation}
  where $ \ell(\epsilon) =\max \left\{
    L_f \mathrel{}:\mathrel{} f\in \C\left(E\cup D[E],\psi(\epsilon)\right)
    \right\}$
    and $\psi(\epsilon) = \omega^{-1}_{D|_E}(\epsilon/4)/4$.
\end{theorem}
Whereas the classical universal approximation results of neural networks consider when a map $\R^M \to \R^M$ can be approximated, Theorem \ref{thm:sample_bounds} establishes which $M$ are sufficient for the existence of a neural network which approximates a non-linear operator $D$ arbitrarily well, uniformly over the domain of $D$. Essentially, the theorem shows that uniform approximation occurs when the domain of $\phi \in \X$ (and $D\phi$) is covered by $x_i \in S_X$ (and $y_i \in S_Y$ respectively) with density controlled by the ratio of $\psi(\epsilon)$, the regularity of $D$, and $\ell(\epsilon)$, the regularity of $\X$. For example, restricting the setting to Lipschitz dynamical systems and functions, Theorem \ref{thm:sample_bounds} lets us lower-bound the minimal number of input/output units as follows:

\begin{corollary}\label{col:lip} If $D: \L_\lambda(K) \to C(K)$ is any $\Lambda$-Lipschitz non-linear operator on the set of $\lambda$-Lipschitz functions on $K$, $\L_\lambda(K)$, then for every $\epsilon > 0$, there is a finite neural network $G: \R^{M} \to \R^M$ which approximates $D$ in the sense of \eqref{eq:uniform_sampless} where   $M > \left(2 + 32\frac{\diam(K)\lambda}{\Lambda\epsilon}\right)^d$.
\end{corollary}

%% file: sections/related_works.tex
%!TEX root = ../main.tex

\section{Related Work}

The precedent for our results stands on a substantial body of work studying the properties of infinite dimensional neural networks. In particular, \cite{neal} proposed the first analysis of neural networks with countably infinite nodes, showing that as the number of nodes in discrete neural networks tends to infinity, they converge to a Gaussian process prior over functions. Later, \cite{williams1998computation} provided a deeper analysis of such a limit on neural networks. A great deal of effort was placed on analyzing covariance maps associated to the Guassian processes resultant from infinite neural networks. These results were based mostly in the framework of Bayesian learning, and led to a great deal of analyses of the relationship between non-parametric kernel methods and infinite networks, including \cite{le2007continuous}, \cite{seeger2004gaussian}, \cite{cho2011analysis}, \cite{hazan2015steps}, and \cite{globerson2016learning}.

The origin of the functional, operator, and basis layer types of Definition \ref{def:layer_types} spurred directly out of this initial work. Specifically, \cite{hazan2015steps} define hidden layer \emph{infinite layer neural networks} with one or two layers which map a vector  $x \in \mathbb{R}^n$ to a real value by considering infinitely many feature maps $\phi_w(x) = g\left(\langle w, x\rangle\right)$ where $w$ is an index variable in $\mathbb{R}^n.$  Then for some weight function $u: \mathbb{R}^n \to \mathbb{R},$ the output of an infinite layer neural network is a real number $\int u(w)\ \phi_w(x)\ d\mu(w)$. It should be noted that \cite{le2007continuous} present a similar construction. The authors show that this instantiation of a network of the form $\Tf \circ  g \circ \Tb$ is universal. Another variant of infinite dimensional neural networks, which is captured by our layer definitions, is the \emph{functional multilayer perceptron} (Functional MLP). This body of work is not referenced in any of the aforementioned work on infinite layer neural networks, but it is clearly related. The fundamental idea is that given some $f \in V = C(X)$, where $X$ is a locally compact Hausdorff space, there exists a generalization of neural networks which approximates arbitrary continuous bounded functionals on $V$ (maps $f \mapsto a \in \mathbb{R}$). These functional MLPs take the form $\sum_{i=1}^p \beta_i g\left(\int \omega_i(x) f(x)\ d\mu(x)\right)$ which is exactly the composition $\Tn \circ \Tf$ and are universal approximators.  In this context, our results work towards a more complete picture of the universal approximation literature around infinite neural networks by answering the open questions of nonlinear operator and basis map  approximation using different compositions of layer types previously studied in the literature.

As previously mentioned the results of \cite{chen1995universal} show that functions of the form $h: \xi, y \mapsto v^T g(W^Ty + b)\cdot g(\int \xi w\;d\mu) \in \H$ are universal in the family continuous nonlinear operators on compact function space. While this does not show that standard feed-forward two layer neural networks $\G_L$ are universal (due to the multiplication of the non-linearities), the authors show a nice result on $h\ in \H$ which precurses Theorem \ref{thm:sample_bounds}; that is, for every $\epsilon >0$ there are points $S_X = (x_1, \dots, x_M)$ in the domain such that $h(\phi(x_1), \phi(x_2), \dots, y)$ approximates underlying operator uniformly over all functions $\xi$. However, their techniques do not specify both which set of points $S_X$ are sufficient and how large $M$ must be for universal approximation to occur. Theorem \ref{thm:sample_bounds} addresses these issues by showing a stronger claim, namely that finite dimensional $G \in \G_2$ can uniformly approximate nonlinear operators, and providing an exact specification for how large $M$ need be and which conditions on the sample points $S_X$ (covering) are sufficient.

%% file: sections/proofs.tex
%!TEX root = ../main.tex
\vspace{-0.5cm}
\section{Proofs}\label{sec:proofs}
The proofs of these results can be distilled down to three major steps. First, we study how maps of the form $F: \X\to \Y$ can be decomposed into finite-dimensional maps $\tilde F: V \subset \R^N \to \R^M$ through objects called \emph{sample factorizations} which behave similarly to functors.  Second, we construct a neural network $N$ which approximate $\tilde F$ using standard universal approximation techniques. Third, we showing that the different post/pre-compositions of the sample factorizations with layers of $N$ are approximateable using layer types of Definition \ref{def:layer_types}. Then, we prove the main results by showing that post/pre-compositions approximate the desired $F$ uniformly by virtue of the decomposition in the first step.

\paragraph{Notation: Approximately Commutative Diagrams.} In the following sections we will repeatedly be asserting whether or not several pairs of maps composed with various other maps are approximately the same. To simplify the proofs, we introduce the following notation. Let $\scriptd: I \to \mathsf{Met}$ be a diagram of metric spaces $((M_i, d_i))_{i\in I}$ and continuous maps $\scriptd(i \to j) \subset C(M_i, M_j)$ between them  indexed by a "graph" (category) $I$. If for all pairs of commutative paths in $I$ their respective functions  $f_1, f_2: M_i \to M_j \in \scriptd(i \to j)$ in the diagram have $\sup_{x \in M_i} d_j(f_1(x), f_2(x)) <\delta$ then we say $\scriptd$ is a \emph{$\delta$-approximate commutative diagram}. Pictorially, $\scriptd$ is shown as a standard commutative diagram adjacent to the symbol $\natural_{\delta}$ (e.g. in Definition \ref{def:sample_factorization}). When $\delta = 0$ a diagram commutes normally (the maps associated to the paths are equal) and this is denoted $\natural_0$.

\subsection{Sample Factorizations} \label{sec:sample_factorizations}

In general, our goal is to reduce the complexity of approximating a map of the form $F: \X\to \Y$ to that of a finite-dimensional one $\tilde F$ (which we will then approximate using a normal neural network). When $\X$ and $\Y$ are function spaces one method for doing this is by first sampling an input function $f \in \X$ at a finite number of points, then sampling the function $F[f]$ at a finite number of points, and then approximating $F$ by how it transforms these input samples to output samples. In the following section, we propose an abstract notion of this finite dimensional 'sampling' procedure called \emph{sample factorization}, which applies to any  metric space $\X$. We then characterize the conditions under which such a procedure has uniform guarantees, and further what properties of this procedure allow us to construct $\tilde F$ for a wide variety of spaces.

\begin{definition}\label{def:sample_factorization} A sample factorization with error $\delta$ of order $M$ for a metric space $\X$ is a pair of uniformly continuous, linear maps $(\Delta, \Delta^*)$ such that the following two diagrams commute approximately and normally respectively:
  \begin{equation}\label{eq:sample}
    \text{{(a)}}
    \begin{tikzcd}
       \X \arrow{r}{\Delta} \arrow{rd}[swap]{\id_{\X}} &\R^M \arrow{d}{\Delta^*} \\
      & \X
    \end{tikzcd} \natural_{\delta},\;\;\;\;\text{and}\;\;\text{(b)}
    \begin{tikzcd}
      \X  \arrow{d}[swap]{\Delta} & \R^M \arrow{l}[swap]{\Delta^*} \arrow{ld}{\id_{\R^M}} \\
       \R^M
    \end{tikzcd} \natural_0.
  \end{equation} 
  We adopt the notation $\Deltab = \Delta^* \circ \Delta$ and $\Deltab^\dagger = \Delta \circ \Delta^*$. If $\X$ has sample factorizations for all $\delta >0$ we say $\X$ is sample factorizable.
\end{definition}
In the Definition \ref{def:sample_factorization} above, one can think of $\Delta$ as taking finitely many samples of some $f \in \X$ and $\Delta^*$ as constructing some `nonparametric' estimate of $f$ from its samples. Hence \eqref{eq:sample}(a) says that the reconstruction error is uniformly small. Likewise \eqref{eq:sample}(b) says that sampling from a nonparametric estimate constructed from some points yields exactly the points from which it was constructed.

Of key interest to us is that sample factorizations allow one to naturally factor a map $F: \X\to \Y$ to an approximate one between finite dimensional vector spaces. The following proposition shows that sample factorizations are `functorial' in nature. Let $\delta_F(\epsilon) = \omega_F(\epsilon)^{-1}$ denote the inverse modulo of continuity for a uniformly continuous map $F$.  
\begin{proposition}[Map Factorization]\label{prop:map_factorization}
  Let $\X$ and $\Y$ be sample factorizable spaces and fix an absolutely continuos map  $F: \X \to \Y$. Then for any $\epsilon > 0$ take $(\Delta_\X, \Delta^*_\X)$ and $(\Delta_\Y, \Delta^*_\Y)$ to be sample factorizations  of error $\delta_F(\epsilon/\|\Deltab_\Y\|_{op})$ and $\epsilon$ and order $M_1$ and $M_2$. Then the following diagram approximately commutes
  \begin{equation}\label{eq:map_factorization}
    \begin{tikzcd}
      \X 
        \arrow[dashed]{rrr}{|F|}
        \arrow{rd}{\Delta_\X}
      &
      &
      & \Y\arrow{dd}{\id_\Y}\\
      &\R^{M_1} 
        \arrow[dashed]{r}{\tilde F}
        \arrow{dl}{\Delta^*_\X}
      &\R^{M_2} 
        \arrow{ru}{\Delta^*_\Y}
      &\\
      \X
        \arrow{uu}{\id_\X}
        \arrow{rrr}{F}
      &
      &
      &\Y \arrow{lu}{\Delta_\Y}
    \end{tikzcd}\natural_\epsilon.
  \end{equation}
  where $\tilde F$ and $|F|$ are defined by taking the natural paths. Hence $F$ factors into a uniform approximation $\sup_{f\in \X}d_\Y(|F|(f),F(f)) < \epsilon$. 
\end{proposition}
\begin{proof}
   We simply chase the diagram above. Fix an $f \in \X$, by the definition of sample factorization \eqref{eq:sample}(a) we know that $ d_\X(\Delta_\X[f], \id_\X[f]) < \delta_K(\epsilon/\delta_{\Delta_{\Y}}(\epsilon))$ uniformly. Then by uniform continuity of $F$, $d_\Y(F \circ \Delta_\X[f], F \circ \id_\X[f]) < \epsilon/\delta_\Y(\epsilon)$. Finally, by uniform continuity of $\Delta_\Y$ we have 
   \begin{equation}
    d_\Y(|F|[f], F[f]) =  d_\Y(\Delta_\Y\circ F \circ\Delta_\X[f],  \id_\Y\circ F\circ\id_\X[f]) < \epsilon.
   \end{equation}
\end{proof} 
As we will use centrally in our proof of \ref{thm:operator}, the approximately commutative diagram \eqref{eq:map_factorization} of Proposition \ref{prop:map_factorization} guarantees that if one can approximate $\tilde F$ uniformly over $\R^{M_1}$ then one can reconstruct $F$ uniformly over $\X$.

We now show that sample factorizations exist under mild assumptions on $\X$, and provide a lower-bound on the dimensionality $M$ given a desired error $\delta$. For $\epsilon>0$ let $\C(S, \epsilon)$ denote a smallest possible $\epsilon$-cover of some subspace $S$ of metric space by $\epsilon$-balls. Finally for $f\in \X(K)$ let $L_f$ be its Lipschitz parameter (or infinity if it is not defined).
\begin{lemma} \label{lem:sample_conditions} Suppose that both $K \subset \R^d$ and $\X(K)$ are compact where $\X(K)$ is the subset of the continuous real valued functions on $K$ endowed with the uniform topology. Let $\delta >0$, there exists a $\delta$-sample factorization $(\Delta, \Delta^*)$ for $\X(K)$ of order $M = |S|$ where
  \begin{equation}\label{eq:cover}
    S = \C\left(K, 
      \frac{\psi(\delta)}{
        \max_{f\in \C\left(\X(K),\psi(\delta)\right)} L_f
        }    \right),\;\text{and}\; \psi(\delta) = \frac{\delta}{2(1 + c)}. 
  \end{equation}
  Hence for all $f \in \X(K)$ the $\norm{f - \Deltab f} < \delta$.
\end{lemma}
\begin{proof}
  Let $M$ be as above, and define $\Delta: \X(K) \to \R^M$ such that $\Delta: f \mapsto (f(x))_{x \in S}$. Then for any continuous, linear $\Delta^*$ satisfying $\Deltab^\dagger = \id_{\R^{M}}$ and $\norm{\Deltab}_{op} \leq c$, we claim that $(\Delta, \Delta^*$) satisfies the lemma. To see this take any $f \in \X(K)$, then there exists a center $f_n$ in a fixed minimal $\psi(\delta)$-cover of $\X(K)$ such that 
  \begin{equation}\label{eq:unifj}
  \begin{aligned}
    \norm{f - \Deltab f}_{\X} &\leq \norm{f - f_n }_\X + \norm{f_n - \Deltab f_n}_\X + \norm{\Deltab f_n -\Deltab f}_\X \\
    & < \psi(\delta) +  \norm{f_n - \Deltab f_n }_\X  +  \norm{\Deltab}_{op}\psi(\delta) \\
    & \leq \psi(\delta)(1 + c) + \max_{f' \in \C(\X(K), \psi(\delta))}  \norm{f' - \Deltab f' }_\X.
  \end{aligned}
\end{equation}
  Then for any $f \in \X(K)$ we wish to bound $\norm{f - \Deltab f}_\X$. Note that by compactness of $K$ all $f'$ are uniformly continuous and therefore have finite Lipschitz parameters $L_f$. Then any $x \in K$ there is a center $x_n \in S$ such that 
  \begin{equation}\label{eq:asdf}
    \begin{aligned}
    \absn{\Deltab[f](x) - f(x)}  &\leq \absn{\Deltab[f](x) - \Deltab[f](x_n)} \\
    &\;\;\;\; + \underbrace{\absn{\Deltab[f](x_n) - f(x_n)}_{K}}_{(\star)} + \absn{f(x_n) - f(x) }
  \end{aligned}
  \end{equation}
  By construction of $(\Delta, \Delta^*)$, \eqref{eq:sample}(b) holds and thus
  \begin{equation}\label{eq:diag4}
    \begin{tikzcd}
      \X(K) \arrow{r}{\Delta} & \R^M \arrow{r}{\Delta^*} \arrow{rd}[swap]{\id_{\R^M}} & \X(K) \arrow{d}{\Delta} \\
      &&\R^m
    \end{tikzcd} \natural_0.
  \end{equation}
  Both terms in $(\star)$  are exactly the paths in \eqref{eq:diag4}, so $(\star) = 0$.  Therefore, bounding the first term in \eqref{eq:asdf} using the operator norm of $\Deltab$, we have
  \begin{equation}\label{eq:unifk}
    \begin{aligned}
    \absn{\Deltab[f](x) - f(x)}  &\leq  (1 + c)L_f\|x - x_n\|_K \leq \frac{(1 + c)L_f\psi(\delta)}{\max_{g' \in \C(\X(K), \psi(\delta))}  L_{g}}
  \end{aligned}
  \end{equation}
  Combining the uniform bound of \eqref{eq:unifk} with \eqref{eq:unifj} we yield
  \begin{equation}
    \begin{aligned}
      \norm{f - \Deltab f}_{\X}  < 2(1+c)\psi(\delta) = \delta.
    \end{aligned}
  \end{equation}
\end{proof}
We note that one can obtain a tighter bound than what is given above by controlling the reconstruction error $|\Deltab[f](x) - f(x)|$ for specific choice of $\Delta^*$. For example when $\Delta^*$ performs spline estimation and assumptions are made on the smoothness of $f \in \X$, $M$ can be improved.

\subsection{Proof of Main Results}

As aforementioned, our second step is to show that non-linear operator has a map decomposition which is approximated by a finite neural network. Fortunately, this follows directly from the classical universal approximation results and Lemma \ref{lem:sample_conditions}.

\begin{theorem}[Neural Map Factorization]\label{thm:neural_map_factorization} If $F: C(K) \to C(K')$ is some continuous operator and $E \subset C(K)$ compact, then for any $\epsilon > 0$, there exists sample factorizations $(\Delta_{G}, \Delta_{G}^*)$ and  $(\Delta_{E}, \Delta_{E}^*)$ of error $\epsilon/2$ and $\omega^{-1}_{F|_E}\left(\epsilon/(2\|\Delta_{F[E]}\|_{op}\right)$ for $E$ and $F[E]$. Then there exists finite dimensional neural network $N$ with affine maps $\Tn_1$ and $\Tn_2$ such that the following diagram approximately commutes:
\begin{equation}\label{eq:neural_map_factorization}
  \begin{tikzcd}
    E \arrow{d}{F} \arrow{r}{\red{\Delta_E}}\arrow[phantom]{dr}{\natural_\epsilon} & \R^{M_1} \arrow[dashed]{d}{N} \arrow[phantom]{dr}{\natural_0} \arrow{r}{\red{\Tn_1}} & \R^{N} \arrow{d}{g}\\
    F[E] & \R^{M_2} \arrow{l}{\blue{\Delta_{F[E]}^*}} &  \R^{N} \arrow{l}{\blue{\Tn_2}} 
  \end{tikzcd}\natural_\epsilon.
\end{equation}
\end{theorem}

\begin{proof}
Denote $G = F[E]$. Since $E \subset C(K)$ compact, Lemma \ref{lem:sample_conditions} implies that both $E$ and $G$ are sample factorizable. Let $(\Delta_{G}, \Delta_{G}^*)$ and  $(\Delta_{E}, \Delta_{E}^*)$ be sample factorizations of error $\epsilon/2$ and $\omega^{-1}_{F|_E}\left(\epsilon/(2\|\Delta_{G}\|_{op}\right)$ for $E$ and $G$ respectively. Then by Proposition \ref{prop:map_factorization}, the associated maps $\tilde F$ and $|F|$ are such that the diagram \eqref{eq:map_factorization} approximately commutes with error $\epsilon/2$. 
  
  Since $\Delta_E[E] \subset \R^{M_1}$  is compact by continuity of $\Delta_E$, the universal approximation theorem of \cite{stinchcombe1999neural} implies that there exists finite-dimensional neural network $N:\R^{M_1} \to \R^{M_2}$ with $h$ hidden units such that the following diagram approximately commutes:
  \begin{equation}
    \begin{tikzcd}
      \R^{M_1} \arrow{r}{\tilde F}\arrow{d}{\id_{\R^{M_1}}} & \R^{M_2} \\
      \R^{M_1} \arrow{r}{N} & \R^{M_2} \arrow{u}[swap]{\id_{\R^{M_2}}}
    \end{tikzcd}
    \natural_{\epsilon/2}.
  \end{equation}
  Combining this with the bottom half of diagram \eqref{eq:map_factorization} shows that the following diagram approximately commutes: 
  \begin{equation}
    \begin{tikzcd}
      E 
        \arrow{r}{\Delta_E}
        \arrow{d}{F}
        \arrow[phantom]{dr}{\natural_{\epsilon/2}}
      & \R^{M_1} 
        \arrow{r}{\id_{\R^{M_1}}}
        \arrow{d}{\tilde F}
        \arrow[phantom]{dr}{\natural_{\epsilon/2}}
      & \R^{M_1} \arrow{d}{N}\\
      G & \R^{M_2} \arrow{l}{\Delta_G^*} & \R^{M_2} \arrow{l}{\id_{\R^{M_2}}} 
    \end{tikzcd}
    \natural_{{\pmb \epsilon}}.
  \end{equation}
  Therefore the map given by $P = \Delta_G^* \circ N \circ \Delta_E = {\blue{\Delta_G^* \circ \Tn_2}} \circ g  \circ {\red{\Tn_1 \circ \Delta_E}}$
  approximates $F$ uniformly with error $\epsilon.$  
\end{proof}

Now that we have shown that the map $P$ uniformly approximates the desired $F$, a viable way to deduce the results of the main theorems through the following strategy. Observe that the red and blue maps functions in the composition are affine so if we can construct layers of the type in Definition \ref{def:layer_types}, that approximate them then we are done.

To that end we show that there are functional, operator, and basis map layers which can approximate the composition (and precomposition) of a sample factorization and an affine map. Let $(\Delta_E, \Delta_E^*)$ be a sample factorization of some compact $H \subset C(K)$ of order $M$.  
\begin{lemma}[Layer Approximation]\label{lem:layer_approximation}
  Let $\X \subset C(K),  \Y \subset C(K'')$ be compact with sample factorizations $(\Delta_\X, \Delta_\X^*)$ and $(\Delta_\Y, \Delta_\Y^*)$ of order $M_1, M_2$ and error $\epsilon_1, \epsilon_2$ respectively. Then if $\Af: \R^{M_1} \to \R^{M_2}$ is an affine map,then for all $\epsilon > 0$ there exist functional, operator, and basis map layers $T^f, T^o, T^b$ such that the following diagram approximately commutes:
  
  \begin{equation}\label{eq:layer_approx}
    \begin{tikzcd}
      \X 
        \arrow[dashed]{rr}{T^o}
        \arrow{dr}{\Delta_\X}
        \arrow[dashed]{ddr}[swap]{T^f}
      && \Y \\
      &\R^{M_1}
        \arrow{d}{\Af}
        \arrow[dashed]{ru}{T^b}& \\
      & \R^{M_2}
        \arrow{ruu}[swap]{\Delta_\Y^*} &
    \end{tikzcd}\natural_\epsilon
  \end{equation}
\end{lemma}
The proof of this lemma is in Appendix \ref{proof:layer}. Now all that remains is to combine the diagrams of the previous lemma and that of the Theorem \ref{thm:neural_map_factorization}.

\begin{proofsubsubsection}{Theorem \ref{thm:operator}, \ref{thm:fb}, and \ref{thm:basis}}\label{sec:proof_basis}
  Observe that by Lemma \ref{lem:sample_conditions} any compact subset, $Z \subset C(K'')$ is sample factorizable (and for any $\epsilon'$ there exists sample factoizations $(\Delta_Z, \Delta_Z^*)$). By the approximately commutative diagrams of Theorem \ref{thm:neural_map_factorization}, Lemma \ref{lem:layer_approximation}, and Definition \ref{def:sample_factorization}(b), the following diagram approximately commutes:
  \begin{equation}\label{eq:proof_all}
    \begin{tikzcd}
      E \arrow[bend left=40,dashed]{rrr}{\To_1} \arrow{r}{\red{\Delta_E}}\arrow{ddd}{F} \arrow[dashed]{drr}{\Tf} & \R^{M_1} \arrow{r}{\red{\Tn_1}}& \R^N \arrow{d}{\id} \arrow{r}{\Delta_Z^*} \arrow[phantom]{dr}{\natural_0} & Z \arrow{d}{\id} \\
      & & \R^N \arrow[phantom]{dr}{\natural_0} \arrow{d}{g}& Z \arrow{l}{\Delta_Z} \arrow{d}{g} \\
      & & \R^N  \arrow[phantom]{dr}{\natural_0}\arrow{d}{\id}\arrow[dashed]{lld}[swap]{\Tb} & Z \arrow{d}{\id}\arrow{l}{\Delta_Z} \\
      F[E] & \R^{M^2} \arrow{l}{\blue{\Delta_{F[E]}^*}} & \R^N \arrow{l}{\blue{\Tn_2}} & Z \arrow{l}{\Delta_Z}  \arrow[bend left=40,dashed]{lll}[swap]{\To_2}
    \end{tikzcd}\natural_{\pmb{\epsilon}}
  \end{equation}
  In particular, we construct the upper right hand $\natural_0$ commutative square of \eqref{eq:proof_all} by composing the diagram in \eqref{eq:sample}(b) with the diagram in \eqref{eq:neural_map_factorization} after $\Tn_1$ and before $g$. Then by composing with that square with $g$ followed by $\id$ (moving downward), the right hand side of \eqref{eq:proof_all} commutes approximately (in fact, normally) with the non-dashed lines. Then for each triangle composed of a dashed line and a solid line, (for example $\To_1$ and $\Delta_Z^* \circ \red{\Tn_1} \circ \red{\Delta_E}$) the existence of the dashed line (in this example $\To_1$) for any $\epsilon'$ follows directly from diagram \eqref{eq:layer_approx} in Lemma \ref{lem:layer_approximation}. Hence the whole diagram \eqref{eq:proof_all} approximately commutes.

  To see Theorem \ref{thm:operator} note that \eqref{eq:proof_all} implies there exist $\To_1$, $\To_2$ such that $\|F - \To_2 \circ g \circ \To_1\| < \epsilon$ uniformly. Likewise for Theorem \ref{thm:fb} note that \eqref{eq:proof_all} implies there exist $\Tf$, $\Tb$ such that $\|F - \Tb \circ g \circ \Tf\| < \epsilon$ uniformly. Finally to see Theorem \ref{thm:basis}, observe that \eqref{eq:proof_all} and the property  \eqref{eq:sample}(a) of sample factorizations implies there exist $\Tn$, $\Tb$ such that $\|F \circ \Delta_E^* - \Tn \circ g \circ \Tb\| < \epsilon$.
\end{proofsubsubsection}

\begin{proofsubsubsection}{Theorem \ref{thm:sample_bounds} }
  Note that $\mathsf{Interp}$ and the sampling procedure $\phi \mapsto (\phi(x_i))_{i = 1}^M$ meet the conditions to be a sample factorization, and in particular they have order and error exactly sufficient for both sample factorizations $(\Delta_E, \Delta_E^*)$ and $(\Delta_{F[E]}, \Delta_{F[E]}^*)$ when $F = D$ in notation. Hence by Theorem \ref{thm:neural_map_factorization} there exists a $G \in \G_2$ which satisfies the conditions of the theorem. Finally by taking the path in \eqref{eq:proof_all} $\Tn_i$ containing the theorem follows.
  \end{proofsubsubsection}

  \begin{proofsubsubsection}{Corollary \ref{col:lip}}
    Recall that if $\diam(K)$ denotes the diameter of $K$, then $K$ is contained in a ball of finite radius $\diam(K)/2$ by compactness. By Example 5.5 of \cite{wainwright2019high} the covering number of a ball of radius $\diam(K)/2$ by $\gamma$-balls is upper bounded by $(2 + 2\diam(K)/\gamma)^d$. Since $\L_\lambda(K)$ is compact, we can apply Theorem \ref{thm:sample_bounds} where that $\psi(\epsilon) = \epsilon 16/\Lambda$ and $\ell(\epsilon) = \lambda$. Setting $\gamma = \psi(\epsilon)/\ell(\epsilon)$ this proves the result.
    
    \end{proofsubsubsection}

%% file: sections/appendix.tex
%!TEX root = ../main.tex
\appendix

\section{Proofs of Technical Lemmas}

\begin{lemma}\label{lem:linear_operator}
  Suppose $K, K'$ are $\sigma$-compact, locally compact, measurable, Hausdorff spaces. If $Q: C(K) \to C(K')$ is a bounded linear operator then there exists a  Borel regular measure $\nu$ and a weak$^*$ continuous family of $L^1(\nu)$ functions $W(t,s) = W_t(s) \in L^1(\nu)$ on $K$ (and hence $K'$) 
  such that $Q[y^\ell](s) = \int_{K} y^\ell(s)W(t,s)\ d\nu(s)$ for all $y^\ell \in C(K)$.
  \end{lemma}
  \begin{proof}
    Let $\zeta_t :C(K')\to \mathbb{R}$ be a linear form which evaluates its arguments at $t\in K'$; that is, $\zeta_t(f) = f(t)$.  Then because $\zeta_t$ is bounded on its domain, $\zeta_t\circ Q = Q^\star\zeta_t: C(K) \to \mathbb{R}$ is a bounded linear functional. Then from the Riesz Representation Theorem we have that there is a unique regular Borel measure $\mu_t$ on $K$ such that 
  \begin{equation}
  \begin{aligned}
      \left(Qy^\ell\right)(t) = Q^\star \zeta_t\left(y^\ell\right) &= \int_{K} y^\ell(s)\ d\mu_t(s), \\
      \norm{\mu_t} &= \norm{Q^\star \zeta_t} 
  \end{aligned}
  \end{equation}
  
  We will show that $\kappa: t \mapsto Q^\star \zeta_t$ is continuous. Take an open neighborhood of $Q^\star \zeta_t$, say $V \subset [C(K)]^*$, in the weak* topology. Recall that the weak* topology endows $[C(K)]^*$ with smallest collection of open sets so that maps in $i({C(K)}) \subset [C(K)]^{**}$ are continuous where $i: C(K) \to [C(K)]^{**}$ so that $i(f) = \hat{f} = \phi \mapsto \phi(f), \phi \in [C(K)]^*$. Then without loss of generality 
  $$V = \bigcap_{n=1}^m \hat{f}_{\alpha_n}^{-1}(U_{\alpha_n})$$
  where $f_{\alpha_n} \in C(K)$ and $U_{\alpha_n}$ are open in $\mathbb{R}.$ Now $\kappa^{-1}(V) = W$ is such that if $t \in W$ then $Q^\star \zeta_t \in \bigcap_1^m \hat{f}_{\alpha_n}^{-1}(U_{\alpha_n})$. Therefore
  for all $f_{\alpha_n}$ then $Q^* \zeta_t(f_{\alpha_n}) = \zeta_t(Q[f_{\alpha_n}]) = Q[f_{\alpha_n}](t) \in U_{\alpha_n}.$ 
  
  We would like to show that there is an open neighborhood of $t$, say $D$, so that $D \subset W$ and $\kappa(Z) \subset V$. First since all the maps $Q[f_{\alpha_n}]: K' \to \mathbb{R}$ are continuous let $D = \bigcap_1^m (Q[f_{\alpha_n}])^{-1}(U_{\alpha_n}) \subset K'$. Then if $r \in D$, $\hat{f}_{\alpha_n}[ Q^\star \zeta_r] = K[f_{\alpha_n}](r) \in U_{\alpha_n}$ for all $1 \leq n \leq m$. Therefore $\kappa(r) \in V$ and so $\kappa(D) \subset V$.
  
  As the norm $\norm{  \cdot }_*$ is continuous on $[C(K)]^*$, and $\kappa$ is continuous on $K'$, the map $t \mapsto \norm{\kappa(t)}$ is continuous. In particular, for any compact subset of $K'$, say $F$, there is an $r \in F$ so that $\norm{\kappa(r)}$ is maximal on $F$; that is, for all $t \in F$, $\norm{\mu_t} \leq \norm{\mu_r}.$ Thus $\mu_t \ll \mu_r.$

  Now we must construct a Borel regular measure $\nu$ such that for all $t \in K',$ $\mu_t \ll \nu$. To do so, we will decompose
  $K'$ into a union of infinitely many compacta on which there is a maximal measure. Since $K'$ is a $\sigma$-compact locally compact Hausdorff space we can form a union $K' = \bigcup_1^\infty U_n$ of precompacts $U_n$ with the property that $U_n \subset U_{n+1}.$ For each $n$ define $\nu_n$ so that $\chi_{U_{n} \setminus U_{n-1}}\mu_{t(n)}$ where $\mu_{t(n)}$ is the maximal measure on each compact $cl(U_n)$ as described in the above paragraph. Finally let $\nu = \sum_{n=1}^\infty \nu_n.$ Clearly $\nu$ is a measure since every $\nu_n$ is mutually singular with $\nu_m$ when $n \neq m$. Additionally for all $t \in K'$, $\mu_t \ll \nu$.
  
  Next by the Lebesgue-Radon-Nikodym theorem, for every $t$ there is an $L^1(\nu)$ function $W_t$ so that $\ d\mu_t(s) = W_t(s)\ d\nu(s)$. Thus it follows that
  \begin{equation}
  \begin{aligned}
    Q\left[y^\ell\right](t) &= \int_{K} y^\ell(s)W_t(s)\ d\nu(s) \\
  &= \int_{K} y^\ell(s)W(t,s)\ d\nu(s)
  \end{aligned}
  \end{equation}
  This completes the proof.   
  \end{proof}

  \begin{proofsubsection}{Lemma \ref{lem:layer_approximation}} \label{proof:layer}
    We prove the result by showing that each triangle in \eqref{eq:layer_approx} commutes approximately. 
  
    \emph{Step 1:} 
    $T^f$. Let $\epsilon >0$ be given, and without loss of generality let $\Af$ be a linear map.  If $\Delta_X$ is of the type in Lemma \ref{lem:sample_conditions}, then $\Af \circ \Delta_H$ is a bounded linear functional since $\X$ is compact. By the Riesz-Markov-Kakutani representation theorem there is a vector valued measure $\nu$ such that for all $f \in \X$
    \begin{equation}
      \Af \circ \Delta_\X[f] = \int_K f(u)\;d\nu(u),
    \end{equation}
    In particular, this measure is the linear combination of measures: let $W \in R^{m\times n}$ be the matrix representation of $\Af$, and then $\nu_j = \sum_{x\in S} W_{jx} \delta_{x}$ where $S$ is the cover from \eqref{eq:cover} and $\delta_{x}$ is the Dirac measure. Let $\mu$ denote the Lebesgue measure on $K$ and define
    \begin{equation}
      \upsilon_r^j: E \mapsto \sum_{x\in S} \frac{W_{jx}}{\mu\left(B_r(x)\right)} \int_{B_r(x)} \chi_E(u) \;d\mu(u).
    \end{equation}
    Then clearly $\upsilon_r^j \ll \mu$. Further for a fixed $f$, there exists an $\rho > 0$ such that for all $r < \rho$, 
    $\|\int f\;d\upsilon_r - \int f\;d\nu\| < \epsilon$ by [Theorem 3.18 Folland]. Applying the covering argument of Lemma \ref{lem:sample_conditions}, for any $\psi > 0$, take $\rho(\psi) = \min_{g \in N(K,\psi)} \sup \{r : \|\int g\;d\upsilon_r - \int g\;d\nu\| < \psi\}. $
    \begin{equation}
      \begin{aligned}
        \|\upsilon_{\rho(\psi)}(f) -\nu(f)\| <   \psi(1  + \|W\|_{op} + \|W_{op}\|\|\Delta_\X\|_{op})
      \end{aligned}
    \end{equation}
    Now letting $\psi = \epsilon/(1  + \|W\|_{op} + \|W_{op}\|\|\Delta_E\|_{op})$ we have that $\upsilon_{\rho(\psi)} \to \nu$ in the weak$^*$-topology. So for any $\epsilon$, $T^f$ defined by vector valued Radon-Nikodym derivative $\frac{d \upsilon_{\rho(\psi)}}{d \mu} \in L^1(\mu)$ has $\|T^f - \Af\circ \Delta_\X\| < \epsilon$ uniformly over $H$. One can extend this to affine maps by subtracting and then adding bias terms.

    \emph{Step 2:} $T^o$. 
    Let $\epsilon' > 0$ be given. We again assume that $\Af$ is a linear map.  Then by Step 1 there exists a $T^f$ such that $\|T^f - \Af \circ \Delta_\X\| <\epsilon'$ uniformly. Then composition with $\Delta_\Y^*$ yields that for every $v \in K''$ define $D: \X \to \Y$ as
    \begin{equation}\label{eq:17lol}
      D[f](v) = \Delta_\Y^* \circ T^f[f](v) =  \Delta^{v*}_\Y\left[ \int_K f(u) w(u) d\mu(u)\right] =  \int_K \Delta_\Y^{v*} \left[f(u) w(u)\right] d\mu(u),
    \end{equation}
    where $\Delta_\Y^{v*}: \R^{M_2} \to \R$ is the linear map $f \mapsto \Delta_\Y^{*}[f](y)$. Then $D$ is a bounded, absolutely continuous operator on $\X$. Then by Lemma \ref{lem:linear_operator} there exists a Borel regular measure $\nu$ such that for all $v \in K''$ there is a function  $w_v \in L_1(K, \nu)$ such that $D[f](v) = \int_K f(u) w_v(u) d\nu(u)$ with the additional property that the map $K \ni v \mapsto w_v\ d\nu$ is continuous in the weak$^*$ topology. Then by \eqref{eq:17lol} when $\mu(E) = 0$ $D[\xi_E] =0$ and hence $\nu \ll \mu$. Hence by the Radon-Nikodym theorem there exists a function $\omega_v \in L_1(K, \mu)$ such that $\omega_v\; d\mu = w_v\;d\nu$ with the map $v \mapsto \omega_v\ d\mu$ weak$^*$ continuous. Therefore the map $v \mapsto D[f](v) = \int_K f \omega_v\;d\mu$ is continuous and $\omega = (u,v) \mapsto d\omega_v/d\mu(u)$ is measurable in the associated product measure. Therefore let $T^o[f] = \int_K f(u) \omega(u,v)\;d\mu(u)$ and then $\|T^o - \Delta^*_\Y \circ \Af \circ \Delta_\X\| < \epsilon$ uniformly. One may add a bias term $\Delta^*_\Y[b]$ to show the affine case when $\Af = \Af' + b$. with $\Af'$ linear.
  
    \emph{Step 3:} $T^b$. 
    Let $V = \Af[I_{M_1}]^T$ be the transpose matrix defining $\Af$. Then define $M_1$ functions $\Y \ni w_i: v \mapsto \Delta_\Y^*[V_i](v)$. Then it follows that if $T^b$ is defined as the following inner product, then for all $v \in K''$ and all $x \in \R^{M_1}$
    \begin{equation}
       \langle w(v), x \rangle_{\R^{M_1}} = \langle \Delta_\Y^*[V](v), x\rangle_{\R^{M_1}}
      =\langle \Delta_\Y^*[I_{M_2}](v), V^Tx\rangle_{\R^{M_2}} =  \Delta_\Y^*[Wx](v).
    \end{equation} 
    Hence the upper right triangle commutes normally.
  \end{proofsubsection}